%% file: main.tex
\documentclass{article}

\usepackage{iclr2027_conference,times}

\usepackage{microtype}
\usepackage{graphicx}
\usepackage{subcaption}
\usepackage{paralist}
\usepackage{booktabs} % for professional tables

\usepackage{hyperref}
\usepackage{bbm}
\usepackage{url}
\usepackage{multirow}
\usepackage{multicol}
\usepackage{enumitem}
\usepackage{algorithm}      
\usepackage{algorithmic} 
\usepackage{latexsym}
\usepackage{siunitx}

\newcommand{\tick}{\checkmark}
\newcommand{\cross}{\(\times\)}
\usepackage{amsmath}
\usepackage{amssymb}
\usepackage{mathtools}
\usepackage{amsthm}

\usepackage[capitalize,noabbrev]{cleveref}

\theoremstyle{plain}
\newtheorem{theorem}{Theorem}[section]
\newtheorem{proposition}[theorem]{Proposition}

\theoremstyle{definition}

\theoremstyle{remark}
\newtheorem{remark}[theorem]{Remark}

\title{Can One-Shot Test-Time Data Augmentation Help with Generalization?}

\author{%
  Yunwei Bai \\
  National University of Singapore \\
  \texttt{baiyunwei@u.nus.edu} \\
  \And
  Yao Shu \\
  Hong Kong University of Science and Technology (Guangzhou) \\
  \And
  Ying Kiat Tan \& Tsuhan Chen \\
  National University of Singapore
}

\iclrfinalcopy

\begin{document}
\maketitle
\lhead{Preprint}

\input{sections/1_introduction}
\input{sections/2_relatedworks}
\input{sections/3_method}
\input{sections/4_theory}
\input{sections/5_experiments}
\clearpage
\bibliography{example_paper}
\bibliographystyle{iclr2027_conference}
\clearpage
\appendix
\input{sections/6_appendix}
\end{document}

%% file: sections/1_introduction.tex
\begin{abstract}
Data augmentation is crucial for model generalization, 
but existing methods are mostly 
centered on the training stage. Test-time augmentation, while 
underexplored, can be practically effective for generalization
while avoiding extra model parameters or fine-tuning. 
Given the increasing training cost and the 
literature gap, we study whether it is possible to perform effective
test-time augmentation using image generation from 
just the single original image. 
We first analyze the importance of 
test-time augmentation, and then design and study a simple 
yet natural operator named 1S-DAug, which comprises geometric perturbations 
with controlled noise injection and image-conditioned denoising. 
We obtain positive results on well-established image-classification
benchmarks across four datasets and multiple models, achieving up
to 20\% relative accuracy improvement without model training or parameter access. 
Code will be released.
\end{abstract}

\section{Introduction}
Data augmentation is crucial, 
but most existing methods focus on the training stage \citep{Liang_2024}, 
which involves considerable computational cost as model sizes continue
to grow \citep{zheng2023cowclip}. At test time,
when each example carries high influence, degradations are especially
harmful, while the model must rely on precise visual cues. 
For test-time augmentation, 
achieving diversity while preserving class-defining content is
therefore central yet non-trivial. Given 
the practical importance and the literature gap, we thus ask the 
\textbf{novel question}:  
How to synthetically augment an image based on just itself? Can the test-time augmentation be effective?

Prior work on test-time augmentation mainly involves traditional
geometric perturbations like rotation and flipping, and  
synthetic test-time data augmentation is underexplored. 
On the one hand, the traditional geometric or photometric 
transforms often add limited new 
information and may degrade image quality \citep{Liang_2024, bai2025fsl}. 
On the other hand, to the best of our knowledge,
we are the \textbf{first} 
to study generative data augmentation under 
a strict test-time context, based on the single input image itself. 
Image-to-image translation with adversarial 
training (e.g., FUNIT, CycleGAN) can be effective in specific domains, 
yet it is prone to training instabilities and to inconsistent quality 
across dissimilar object categories and poses 
\citep{liu2019few,zhu2017unpaired}. Prior attempts to employ GANs for
data augmentation typically restrict their use to training time
\citep{hariharan2017lowshot,schwartz2018deltaencoder,hariharan2017low},  
and \citet{bai2025fsl} 
requires access to training samples, which is 
restricted to same-domain evaluation and has limited applicability across
datasets. Some prior work leverages diffusion-generated
images for training data augmentation. However, the methods rely on 
fine-tuning using text prompts or additional target-class samples 
\citep{Trabucco2024EffectiveDA, He2023IsSyntheticDataReady}, and are thus not
suitable for the one-shot context.

\begin{figure*}
    \centering
    \includegraphics[width=\linewidth, trim={0cm, 2cm, 2cm, 0}, clip]{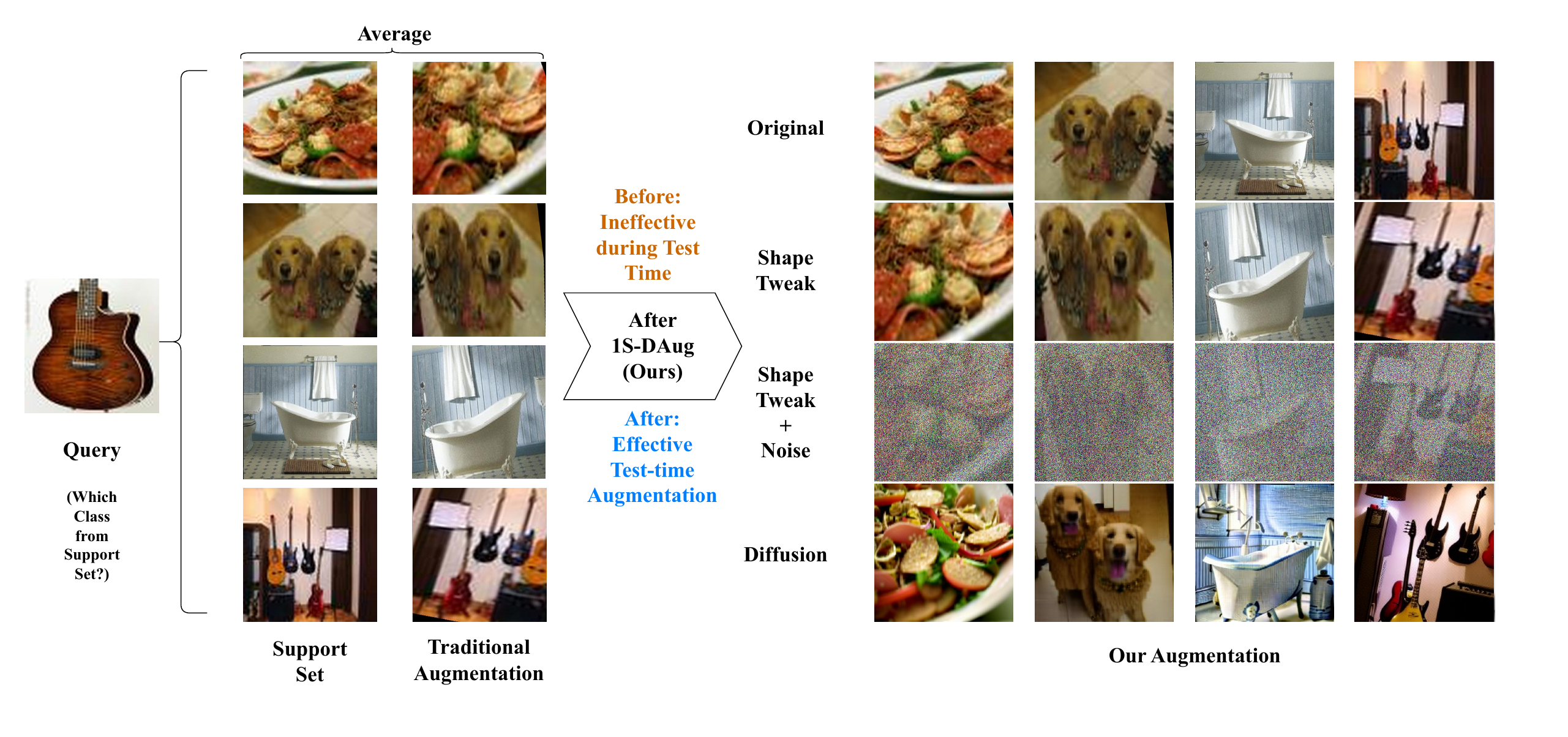}
    \caption{Pipeline of the one-shot test-time augmentation operator. 
    We apply a shape tweak and controlled noise, 
    then perform attention-conditioned diffusion to synthesize class-faithful variants. 
    Features from the original and the generated views are averaged before
    classification. Traditional test-time geometric transformations provide limited diversity
    and may decrease image quality.}
    % \vspace{-1.6em}
    \label{fig:aug_illu}
\end{figure*}

To perform one-shot test-time data augmentation from a single image,
we synthesize a set of 
diverse yet faithful 
variants based on a simple yet novel operator that 
combines three ingredients: 
(i) class-preserving shape perturbations to broaden pose and layout 
coverage; (ii) controlled noise injection to hide non-class-defining 
and distorted details; and (iii) denoising guided by attention-based 
conditioning on the original image so that content-defining cues are 
preserved while appearance and pose vary in a controlled manner. 
The augmentation pipeline does not require information beyond the input 
reference image itself, and the shape tweak and partial image-guided denoising
respectively promote diversity and accuracy, which are crucial for an effective
ensemble \citep{zhou2012ensemble}.
During image classification, the generated variants are aggregated 
with the original representation. The operator is model-agnostic and requires
neither training nor model-parameter access.

To explain the effectiveness in model generalization, from an ensemble 
perspective, test-time augmentation aggregates predictions across 
multiple views of the same input and averages out error; in an 
idealized independence thought experiment, a single-view error 
rate $\varepsilon$ would combine as $\varepsilon^{m}$ for all $m$ 
views being wrong, which is an exponential reduction. 
Complementarily, averaging 
class-faithful views can reduce the influence of view-specific 
variation and contract the effective feature radius with high probability, 
potentially improving the generalization risk bound in expectation.

In summary, our contributions include: 1) We identify the literature gap and
the importance of the one-shot test-time data augmentation; 2) We study
how to perform the one-shot data augmentation and whether it indeed helps
improve model generalization with access to only one test-time data point. We
conduct various experiments on the data operator,
obtaining positive results of up to 20\% 
relative image-classification accuracy improvement, without retraining or
model-parameter access.

%% file: sections/2_relatedworks.tex
\section{Related Work}
\textbf{Data augmentation} has been widely explored, 
while test-time data augmentation is less studied \citep{Liang_2024}. 
In \textbf{few-shot learning} (FSL), there are different setups.
We focus on the traditional image-only inductive setup, where
prior augmentation-based approaches 
expand training diversity via feature hallucination~\citep{hariharan2017low} 
or GAN-driven synthesis~\citep{wang2018low}, but they are 
largely \emph{training-time} and rely on supervision from base classes. 
In contrast, \emph{test-time} augmentation for FSL is harder; 
it needs to generate high-quality, class-faithful variants for 
unseen classes without any retraining or labels. The only explicit 
work in this direction is FSL-Rectifier~\citep{bai2025fsl}, which 
uses FUNIT~\citep{liu2019few} to combine the shape of one image 
with the class-defining style of another. This serves as a proof of 
concept, mainly targeting animal-face datasets with limited applicability to
broader settings. For \textbf{diffusion models}, attention-based
conditioning adapters now inject external signals (e.g., an input image) 
into cross-attention, enabling faithful, controllable 
generation~\citep{mou2023t2iadapter}. 
SDEdit~\citep{meng2022sdedit} adds controlled noise and then denoises
with conditioning. However, a lower noise level yields minimal changes,
whereas a higher level sacrifices faithfulness (e.g., changing the object type),
and is thus unsuitable for data augmentation that requires both diversity
and faithfulness. A few recent works use diffusion-generated synthetic images
for training augmentation outside few-shot learning,
but they rely on fine-tuning using text prompts or additional test-class
samples and are thus unsuitable for test-time augmentation
\citep{He2023IsSyntheticDataReady, Benigmim2023DATUM, Trabucco2024EffectiveDA}. 
More related work is available in Appendix Section~\ref{sec:related}.

%% file: sections/3_method.tex
\section{Test-Time Augmentation Operator}
\label{sec:method}

Diffusion preliminaries and notation used below are provided in
Appendix~\ref{app:diffusion-preliminaries}.

% \paragraph{1S-DAug.}
We produce variants of a single image by (1) applying traditional geometric
changes, (2) injecting a controlled amount of noise to determine edit
magnitude, and (3) denoising via a diffusion process conditioned on the source
image so that content-defining attributes are preserved while details and pose
can vary. We write the resulting single-image augmentation operator as
\begin{equation}
\mathcal{A}(x;\upsilon,\pi)
=
\mathrm{Den}_\varphi\!\Big(
  \mathrm{Noi}_\eta\!\big(T_{\psi}(x)\big)
  \,;\,x,\pi,\lambda_{\mathrm{img}}
\Big)\in\mathcal{X},
\end{equation}
where $\upsilon=(\psi,\eta,\lambda_{\mathrm{img}})$ collects the geometric,
noising, and conditioning hyperparameters, and $\pi$ is an optional text prompt
(omitted when empty); $T_{\psi}$ is a sampled geometric
transform; $\mathrm{Noi}_\eta$ (add noise) applies the forward diffusion
process up to a start time $t_0$ determined by the noise level $\eta$
(e.g.\ via~\eqref{eq:eta-map}); and $\mathrm{Den}_\varphi$ (denoise),
parametrized by fixed diffusion parameters $\varphi$, runs the reverse process
from $t_0$ to $0$ with image-conditioned attention. We next detail each stage
of this operator.

\paragraph{Stage 1: Shape Tweak.}
Changes in pose/layout increase coverage of plausible views without altering
class identity. Let $\psi$ be shape-tweak parameters and
\[
x_{\mathrm{geom}}=T_{\psi}(x),
\]
where $T_{\psi}$ is drawn from a family of traditional image transformations,
including rotations, anisotropic stretches, translations, perspective
jitters, and horizontal flips.

\paragraph{Stage 2: Controlled Noising.}
The noise level sets the change magnitude during the diffusion denoising
pipeline. Lower noise levels emphasize faithfulness; higher noise levels hide
geometric distortion better
and yields more diversity. We use the variance-preserving (VP) forward
kernel~\citep{sohldickstein2015deep,ho2020ddpm,nichol2021improved,meng2022sdedit}.
Let $t\in\{1,\dots,T\}$ index discrete diffusion steps, let
$\beta_t\in(0,1)$ be a variance schedule, and define
$\alpha_t \coloneqq 1-\beta_t$ and
$\bar{\alpha}_t \coloneqq \prod_{\tau=1}^{t}\alpha_\tau$.
Given a user noise level $\eta\in[0,1]$, we choose a start time $t_0=t(\eta)$
(e.g., by matching cumulative noise as in \eqref{eq:eta-map}), and sample
\begin{equation}
x_{t_0}\sim q(x_{t_0}\mid x_{\mathrm{geom}})
=\mathcal{N}\!\Big(\sqrt{\bar{\alpha}_{t_0}}\,x_{\mathrm{geom}},\ (1-\bar{\alpha}_{t_0})\,\mathbf{I}\Big),
\label{eq:vp-forward-geom}
\end{equation}
where $x_{\mathrm{geom}}$ is the geometrically perturbed input (Stage~1),
$\mathbf{I}$ is the identity covariance, and $\mathcal{N}$ denotes a Gaussian
distribution.

\paragraph{Stage 3: Image-Conditioned Diffusion Denoising.}
Let $\mathrm{Enc}$ and $\mathrm{Dec}$ denote the fixed latent encoder and
decoder. For latent diffusion, initialize $z_{t_0}:=\mathrm{Enc}(x_{t_0})$. Let
$h_t\in\mathbb{R}^{N_q\times d_h}$ denote the tokenized hidden state of a U-Net
\citep{ronneberger2015unet} block at time $t$, and let
$f_{\mathrm{img}}(x)\in\mathbb{R}^{L\times d_c}$ and
$f_{\mathrm{txt}}(\pi)\in\mathbb{R}^{M\times d_c}$ be fixed-encoder outputs for
the source image and optional text prompt. Here $N_q$ is the number of U-Net
query tokens, $L$ and $M$ are the image- and text-token counts, $d_h$ and $d_c$
are the corresponding input widths, and $d_k$ is the attention width. With
projection matrices
$W_{Q,t}\in\mathbb{R}^{d_h\times d_k}$ and
$W^{u}_{K,t},W^{u}_{V,t}\in\mathbb{R}^{d_c\times d_k}$ for
$u\in\{\mathrm{txt},\mathrm{img}\}$, define
\begin{equation}
\begin{aligned}
Q_t&=h_tW_{Q,t},\quad
K_t^{\mathrm{txt}}=f_{\mathrm{txt}}(\pi)W^{\mathrm{txt}}_{K,t},\quad
V_t^{\mathrm{txt}}=f_{\mathrm{txt}}(\pi)W^{\mathrm{txt}}_{V,t},\\
K_t^{\mathrm{img}}&=f_{\mathrm{img}}(x)W^{\mathrm{img}}_{K,t},\quad
V_t^{\mathrm{img}}=f_{\mathrm{img}}(x)W^{\mathrm{img}}_{V,t}.
\end{aligned}
\end{equation}
Thus $Q_t\in\mathbb{R}^{N_q\times d_k}$ and the concatenated keys and values
$K_t,V_t\in\mathbb{R}^{(M+L)\times d_k}$. The cross-attention
\citep{vaswani2017attention} is
\begin{equation}
A_t(Q_t,K_t,V_t)=\mathrm{softmax}\!\Big(\tfrac{Q_tK_t^\top}{\sqrt{d_k}}\Big)V_t,
\end{equation}
and we concatenate text/image tokens with a scalar weight
$\lambda_{\mathrm{img}}\ge 0$:
\begin{equation}
K_t=\begin{bmatrix}K_t^{\mathrm{txt}}\\
\lambda_{\mathrm{img}}K_t^{\mathrm{img}}\end{bmatrix},\qquad
V_t=\begin{bmatrix}V_t^{\mathrm{txt}}\\
\lambda_{\mathrm{img}}V_t^{\mathrm{img}}\end{bmatrix}.
\label{eq:attn-concat}
\end{equation}
We set the conditioning variable for the reverse update to
$c_t:=A_t(Q_t,K_t,V_t)$. The VP reverse step \eqref{eq:reverse} then reads
\begin{equation}
z_{t-1} \;=\; \mu_\varphi(z_t,t,c_t)\;+\;\sigma_t\,\epsilon,\qquad \epsilon\sim\mathcal{N}(0,\mathbf{I}),
\end{equation}
where $\mathbf I$ has the latent dimension. Rolling out $t_0\!\to\!0$ produces
$\tilde{x}=\mathrm{Dec}(z_0)$.

%% file: sections/4_theory.tex
\section{Theoretical Insights}
\label{sec:theory}

We provide two perspectives to capture the rationale
behind test-time data augmentation:
(i) a simple risk decomposition
that emphasizes accuracy and diversity and
(ii) a probabilistic analysis of feature-radius
reduction under augmentation.

\subsection{Risk Decomposition into Accuracy and Diversity}

We first illustrate the ensemble effect intuition. We denote
an image with $x$. Let $y\in\{-1,1\}$
indicate a binary label, and let
$D$ denote the resulting distribution
over $(x,y)$. For any real-valued predictor $g$, we use the scaled
squared-loss risk
\begin{equation}
\label{eq:risk-squared-short}
\mathcal{R}(g)
:=
\frac{1}{4}\,\mathbb{E}_{(x,y)\sim D}\bigl[(g(x)-y)^2\bigr],
\end{equation}
which coincides with $0$-$1$ error when $g(x)\in\{-1,1\}$. In particular, if
$g_\theta$ is a sign-valued predictor, the misclassification risk
$R_{\mathrm{cls}}(\theta):=\mathbb{P}_{(x,y)\sim D}
\bigl(y\,g_\theta(x)\le 0\bigr)$ satisfies
$R_{\mathrm{cls}}(\theta)=\mathcal{R}(g_\theta)$. Let $f(x)$ and $f_A(x)$ be
the sign predictors induced by the original and an augmented view,
respectively, and define the two-view ensemble
$\tilde f(x) := \tfrac{1}{2}\bigl(f(x)+f_A(x)\bigr)\in\{-1,0,1\}$. A direct calculation
(Appendix~\ref{app:risk-proof}) yields:

\begin{proposition}[Risk Decomposition]
\label{prop:risk-decomp}
With $\mathcal{R}(\cdot)$ as in \eqref{eq:risk-squared-short},
\begin{equation}
\begin{aligned}
\label{eq:risk-decomp-short}
&\mathcal{R}(\tilde f) - \mathcal{R}(f)\\
=
&\underbrace{\frac{1}{4}\bigl(
\mathbb{E}[f(x)y]-\mathbb{E}[f_A(x)y]
\bigr)}_{\textup{accuracy gap}}
+
\underbrace{\frac{1}{8}\bigl(
\mathbb{E}[f(x)f_A(x)]-1
\bigr)}_{\textup{diversity term}}.
\end{aligned}
\end{equation}
\end{proposition}

Thus improvements come from (i) maintaining or improving single-view accuracy,
and (ii) augmentation diversity.

\subsection{Generalization Bound Improvement}

Denote an image feature by $\Omega$.
Superscripts
$(0)$ and $(1)$ denote original and augmented features,
respectively.

\begin{proposition}[Radius reduction with one augmentation]
\label{prop:radius-reduction-main}
Let $m>1$ and let
$\{\Omega_i^{(0)},\Omega_i^{(1)}\}_{i=1}^{m}$ be drawn from the same continuous
distribution on $\mathbb R^d$. Define
$
\bar{\Omega}_i := \frac{1}{2}\bigl(\Omega_i^{(0)}+\Omega_i^{(1)}\bigr).
$
Let
$r_{\mathrm{orig}}^{\max}:=\max_i \|\Omega_i^{(0)}\|_2$ and
$\hat r^{\max}:=\max_i \|\bar{\Omega}_i\|_2$. Then
$
\mathbb{P}\bigl(\hat r^{\max} < r_{\mathrm{orig}}^{\max}\bigr) > 0.5.
$
\end{proposition}

Averaging more than one augmented view is expected to further stabilize the aggregated
feature radius, but we use only the one-augmentation statement above. An 
assumption-light proof is in
Appendix~\ref{app:radius-reduction}.

We next connect feature radius to a standard complexity bound. Let
$D_\Omega$ be a distribution over feature/label pairs and let
$\mathcal{T}=\{(\Omega_i,y_i)\}_{i=1}^{m}\sim D_\Omega^m$. For $B>0$, we use the
norm-constrained linear class as an analytical readout on these features,
$\mathcal{F}_B:=\{f_w(\Omega)=w^\top\Omega:\|w\|_2\le B\}$. For a loss
$\ell$ that is $L_\ell$-Lipschitz in its first argument, with $L_\ell>0$, define
\[
L_{D_\Omega}(f):=\mathbb{E}_{(\Omega,y)\sim D_\Omega}[\ell(f(\Omega),y)],
\qquad
\widehat L_{\mathcal{T}}(f):=\frac{1}{m}\sum_{i=1}^{m}
\ell(f(\Omega_i),y_i),
\]
and define the empirical Rademacher complexity and its expectation by
\[
\widehat{\mathfrak R}_{\mathcal T}(\mathcal F_B)
:=\mathbb E_\sigma\!\left[
\sup_{f\in\mathcal F_B}\frac1m\sum_{i=1}^m\sigma_i f(\Omega_i)
\right],
\qquad
\mathfrak R_m(\mathcal F_B)
:=\mathbb E_{\mathcal T}\!\left[
\widehat{\mathfrak R}_{\mathcal T}(\mathcal F_B)
\right],
\]
where $\sigma_1,\ldots,\sigma_m$ are independent Rademacher variables. Define
the maximum feature radius
\[
r_{\mathcal T}^{\max}:=\max_i\|\Omega_i\|_2.
\]

\begin{proposition}[Maximum-radius complexity and generalization bounds]
\label{prop:radius-generalization}
For every sample $\mathcal T$,
\begin{equation}
\label{eq:radius-empirical-complexity}
\widehat{\mathfrak R}_{\mathcal T}(\mathcal F_B)
\le \frac{B r_{\mathcal T}^{\max}}{\sqrt m}.
\end{equation}
Consequently,
\begin{equation}
\label{eq:radius-generalization}
\mathbb{E}_{\mathcal{T}}\!\left[
\sup_{f\in\mathcal{F}_B}
\bigl(L_{D_\Omega}(f)-\widehat L_{\mathcal{T}}(f)\bigr)
\right]
\le
2L_\ell\,\mathfrak{R}_m(\mathcal{F}_B)
\le
\frac{2L_\ell B}{\sqrt m}\,
\mathbb{E}_{\mathcal{T}}[r_{\mathcal{T}}^{\max}].
\end{equation}
\end{proposition}

\begin{proof}
For independent Rademacher variables $\sigma_1,\ldots,\sigma_m$,
Cauchy--Schwarz and Jensen's inequality yield
\begin{align*}
\widehat{\mathfrak{R}}_{\mathcal{T}}(\mathcal{F}_B)
&=\frac{1}{m}\mathbb{E}_{\sigma}
\sup_{\|w\|_2\le B}w^\top\sum_{i=1}^{m}\sigma_i\Omega_i \\
&\le\frac{B}{m}
\left(\mathbb{E}_{\sigma}\left\|\sum_{i=1}^{m}
\sigma_i\Omega_i\right\|_2^2\right)^{1/2}
=\frac{B}{m}\left(\sum_{i=1}^{m}\|\Omega_i\|_2^2\right)^{1/2}
\le\frac{B r_{\mathcal{T}}^{\max}}{\sqrt m},
\end{align*}
where the equality uses
$\mathbb{E}_{\sigma}[\sigma_i\sigma_j]=\mathbf{1}\{i=j\}$.
This proves~\eqref{eq:radius-empirical-complexity}. Taking expectation over
$\mathcal{T}$ gives the final inequality in~\eqref{eq:radius-generalization}.
\end{proof}

Let $\mathcal T^{(0)}$ and $\bar{\mathcal T}$ denote samples with features
$\Omega_i^{(0)}$ and $\bar\Omega_i$, respectively, and the corresponding
labels. On the radius-reduction event in Proposition~\ref{prop:radius-reduction-main},
the maximum-radius upper bound on empirical complexity satisfies
\[
\frac{B\hat r^{\max}}{\sqrt m}
<
\frac{B r_{\mathrm{orig}}^{\max}}{\sqrt m}.
\]
Hence the maximum-radius upper bound on empirical Rademacher complexity is
strictly smaller on that event.

%% file: sections/5_experiments.tex
\section{Experiments}
\label{sec:experiment}
% \subsection{Set-up and Main Results}
\paragraph{Can One-Shot Test-Time Data Augmentation Help?}
For all experiments, 
model parameters
remain fixed; we only perform averaging across the original and generated image embeddings. 
We set the aggregation weight of the original image to 0.5 and that of the
additional images collectively to 0.5, thereby
emphasizing the original images. For diffusion noise addition, we
set the noise level to 0.7. Details are available in 
Appendix~\ref{sec:details}.

We choose the classic 5-way-1/5-shot episodic evaluation on 
miniImagenet and tieredImagenet \citep{ILSVRC15} as our main 
evaluation protocols; 
additional experiments are conducted on CUB (fine-grained birds) 
\citep{WahCUB_200_2011} and an animal-face dataset used previously 
for test-time augmentation studies (Animals) \citep{liu2019few, bai2025fsl}. 
% When source benchmarks provide around $84\times84$ inputs (e.g., miniImagenet), we first upsample to the diffusion resolution using a super-resolution/upscaling operator (i.e., Real-ESRGAN) \citep{wang2021realesrgan} to the $512\times512$ scale and then apply our augmentation pipeline. For the main experiments, we scale these images back to the $84\times84$ size to ensure alignment with baseline set-ups. We also compare these results to those without the scale-up procedure, ensuring no discrepancy. 
For the main experiments, 
we retain this established protocol rather than a CLIP-based few-shot protocol
\citep{radford2021learning,kravets2025rethinking}, as its language-derived class
representations introduce a confound.

Formally, for few-shot learning (FSL), the base classes $\mathcal{C}_{\mathrm{base}}$ 
used for training are disjoint from the novel evaluation classes $\mathcal{C}_{\mathrm{novel}}$. An
$N$-way $K$-shot episode samples $N$ novel classes and provides a support set
$\mathcal{S}=\{(x_{c,j},c):c\in\mathcal{C},\,j=1,\ldots,K\}$, followed by
queries whose labels belong to the same episode-specific class set
$\mathcal{C}\subset\mathcal{C}_{\mathrm{novel}}$. The task is to classify each
query among these $N$ classes from the $K$ labeled examples per class
\citep{vinyals2016matching,snell2017prototypical,wang2020generalizing}. The 
evaluation is inductive: each query is predicted using only the support set and
that query, without jointly exploiting other unlabeled queries. 

For FSL, we evaluate four backbones: a shallow 4-layer convolutional 
neural network (ConvNet), a 12-layer residual network (Res12) 
\citep{he2016deep}, a small vision transformer backbone (ViTSmall) 
\citep{dosovitskiy2021an}, and a tiny swin transformer backbone (SwinTiny) 
\citep{liu2021swin}. Note that the encoders ViTSmall and SwinTiny are
pretrained on the ImageNet-1k \citep{ILSVRC15} dataset, while the Res12 
and ConvNet encoders are trained on the training split of each FSL dataset.
Few-shot classification is performed with a non-parametric 
Euclidean-distance/cosine-similarity prototype classifier in the 
episodic setting. These choices match common FSL practice, including 
ProtoNet and FEAT-style set-to-set variants that operate on support/query 
embeddings, ensuring comparability with prior work and our reproduced 
baselines \citep{snell2017prototypical, ye2020fewshot}. We evaluate 
pretrained ProtoNet and FEAT models. We adopt
a Euclidean-distance-based classifier for miniImagenet and tieredImagenet,
and a cosine-similarity-based classifier for Animals and CUB. Our reproduced
baselines follow the official repositories and closely match previously
reported evaluation setups, including DeepEMD \citep{zhang2020deepemd},
Meta-Baseline \citep{chen2021metabaseline}, and MetaOptNet
\citep{lee2019metaoptnet} on 5-way-1/5-shot inductive benchmarks with the Res12
backbone.
Among the baselines, SLA-AG \citep{lee2020slaag} involves self-supervised 
label augmentation, and Meta-MaxUp \citep{ni2020data} involves 
training data augmentation; both are ensemble-based methods.
We sample 15,000 5-way-1/5-shot queries and report mean accuracy 
with 95\% confidence intervals across episodes. This follows the 
standard protocol used in related FSL work 
\citep{ye2020fewshot, snell2017prototypical}. 

\input{tables/mini_tiered}
\input{tables/cub_animals}

\input{tables/vlm_adversarial}

Tables~\ref{tab:main}, \ref{tab:cub}, and~\ref{tab:main2} summarize
5-way-1-shot and 5-way-5-shot results with Res12, ConvNet, ViTSmall, and
SwinTiny backbones. Our variants with 1/2/3 additional augmentations are denoted
as 1S-DAug-1/2/3, respectively. Note that we adopt 5-way-1-shot FSL models
for the 5-way-5-shot evaluation, and Table \ref{tab:cub} contains 
5-way-1-shot results on CUB and Animals. As reported, test-time 1S-DAug 
consistently improves over the corresponding non-augmented baselines and 
over prior strong Res12 methods reported under the same backbone 
(e.g., on miniImagenet, FEAT improves from $63.31\%$ to $69.25\%$, a
$+9.38\%$ relative gain). Note that 
1S-DAug performs significantly better on the Animals dataset compared 
to the prior work that uses GAN reconstruction, which only achieves 
57.90\% on the 5-way-1-shot setup with one augmentation \citep{bai2025fsl}. 
Limitations in the quality of GAN-generated images are discussed in
\citep{bai2025fsl, liu2019few}, and the related failure case visualization 
is available in Appendix Figure \ref{fig:failure}. All gains persist
across the datasets. On the four 5-way-1-shot classification benchmarks,
relative improvements are at least $+4.82\%$ and at most $+17.23\%$ on
miniImagenet, at least $+5.23\%$ and at most $+7.72\%$ on tieredImagenet,
at least $+2.84\%$ and at most $+20.45\%$ on CUB, and at least $+1.63\%$
and at most $+2.73\%$ on Animals. These results imply that
(i) the image-conditioned
diffusion step preserves class-defining content, and (ii) the 
shape tweak with noising creates diversity without compromising faithfulness. 
Additionally, ViTSmall/SwinTiny backbones are pretrained on ImageNet-1k 
\citep{ILSVRC15}, with consistent non-trivial improvement 
over these more knowledgeable backbones, suggesting that the operator is 
applicable to different model architectures, scales, and datasets.

To broaden the evaluation, we also cover
larger vision-language models (VLMs). Large language models have been
explicitly studied for adaptation 
performed at inference time from limited contextual examples 
\citep{brown2020language, min2022rethinking}. Modern VLMs such as LLaVA 
extend this paradigm to image-grounded reasoning \citep{liu2023visual}. 
Our VLM is LLaVA-1.5-7B \citep{liu2024llava15}.
On the adversarial split of POPE 
\citep{li2023pope}, we probe robustness to object 
hallucination under challenging confounders. We aggregate yes/no 
decisions across the original image and multiple generated views. 
As shown in Table~\ref{tab:vlm_adv}, 1S-DAug improves the vanilla baseline
from $78.83\%$/$77.41\%$ to $79.57\%$/$77.49\%$ with two augmented views
($+0.94\%$/$+0.10\%$ relative accuracy/F1 improvement), and to
$81.13\%$/$79.08\%$ with four augmented views ($+2.92\%$/$+2.16\%$
relative accuracy/F1 improvement).
Our implementation, including the prompt and decoding, follows that of Visual
Contrastive Decoding (VCD) exactly \citep{leng2023vcd}. Overall, these 
results affirm the usefulness of test-time augmentation beyond the 
traditional image-classification setting with smaller backbones.
Further VLM implementation details are provided in Appendix~\ref{sec:vlm-details}.

\paragraph{How Much and Where to Augment?}
\label{sec:hp}
In FSL, we study support-only, query-only, and joint support+query augmentation
under ProtoNet (Res12) on miniImagenet. See our ablation table 
(Table~\ref{tab:ablation}) for the full grid. Accuracy improves monotonically 
as we add augmented copies to \emph{both} support and queries
(e.g., from $60.01\%$ with no augmentation to $64.94\%$, 
yielding a $+8.22\%$ relative gain). Adding only support copies while
leaving queries unaugmented can underperform due to distribution mismatch
between prototype construction and query embeddings (e.g., the accuracy with
$+3/0$ is $61.82\%$, well below the $64.55\%$ achieved when queries
are matched with $+3/1$). This suggests that matched augmentation on both 
sides yields the largest benefit. Additionally, we observe decreasing marginal returns as
we add more augmented copies, which is expected. 
\input{tables/vary_shots}

\paragraph{Qualitative Evaluations.}
% We conduct ablation studies for our method, and more analyses, including hyperparameter tuning (Section \ref{sec:hp}), efficiency studies are available in Appendix~\ref{sec:efficiency}.
\begin{figure*}[ht]
    \centering
    \includegraphics[width=\linewidth]{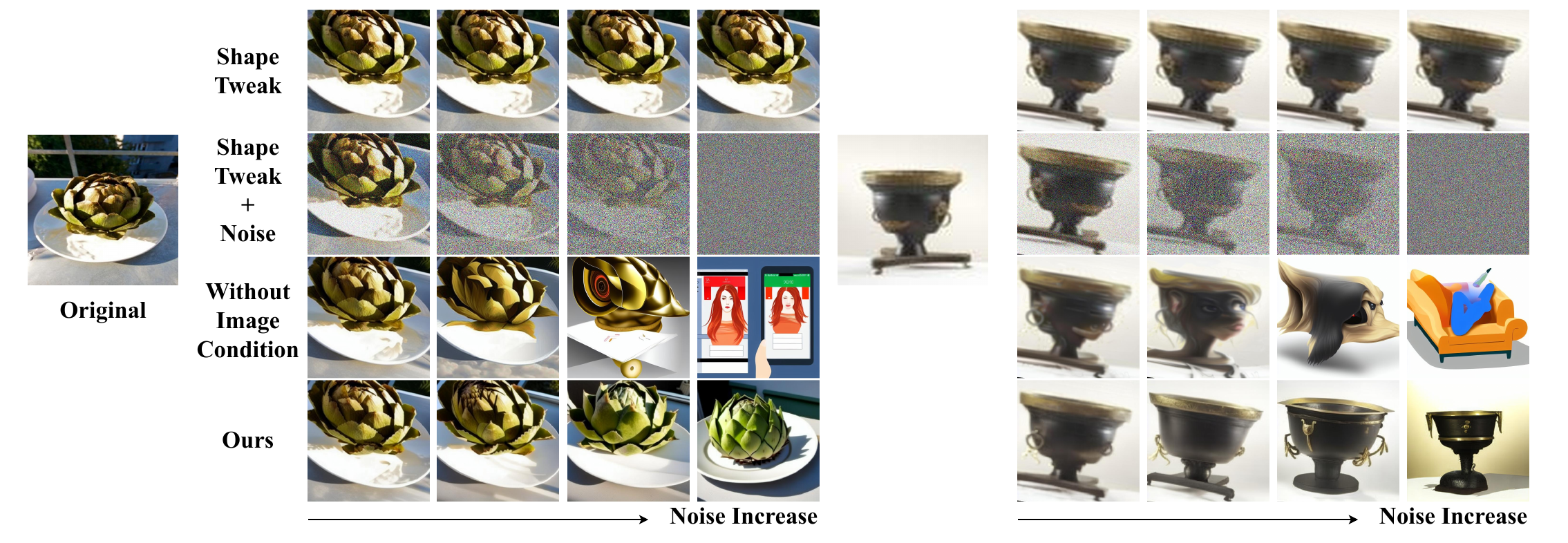}
    \caption{Effect of noise and conditioning. Qualitative ablation on a single input across increasing noise levels. Shape-only edits yield limited diversity; adding noise increases diversity but may reduce fidelity without conditioning. Attention-conditioned diffusion preserves class-defining content while enabling controlled pose/appearance changes; excessive noise without the image condition degrades faithfulness.}
    \label{fig:rm_guide}
    \vspace{-1.0em}
\end{figure*}

\begin{figure*}[ht]
    \centering
    \includegraphics[width=\linewidth]{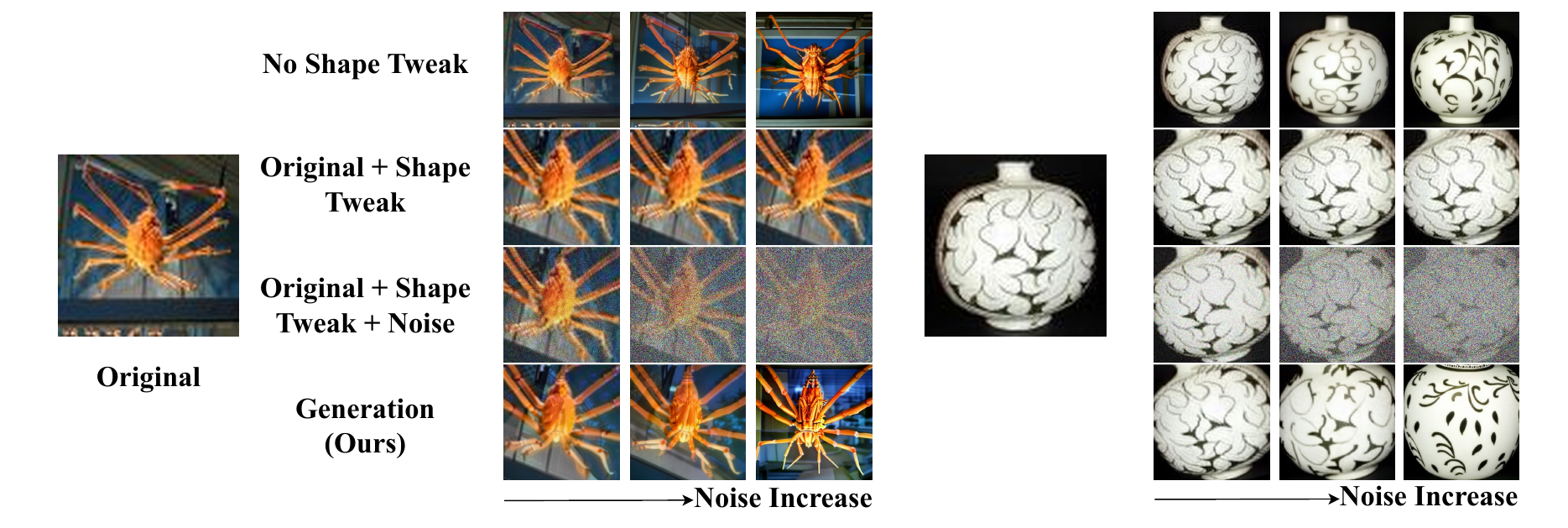}
    \caption{Effect of noise and shape tweak. Comparison across three settings: no shape tweak, shape tweak only, and shape + noise + attention-conditioned diffusion (ours). Increasing noise and including shape tweak expand diversity, and our full setting provides the best balance for both diversity and faithfulness.}
    \label{fig:rm_shape}
    \vspace{-0.4em}
\end{figure*}

Figure~\ref{fig:rm_guide} illustrates the effect of removing image
conditioning or shape tweaks and changing the noise level. With small
noise, changes are minimal; with large noise and no image conditioning, 
generations may drift toward off-class content; removing shape tweaks 
reduces diversity and visible pose/layout variations. More visualizations 
are provided in Appendix~\ref{sec:viz}. 

\paragraph{Additional Ablation Results.}
\input{tables/ablation}
We further dissect the contribution of each component, including 
aggregation weight adjustment, image conditioning, noise magnitude,
shape tweaking, and diffusion generation using FEAT with a Res12 backbone
on miniImagenet, in a controlled 5-way-1-shot setting with one augmented 
query and one augmented support per episode (Table~\ref{tab:shots}). 
We first observe that assigning less aggregation weight to the original samples
slightly reduces model accuracy. This is expected,
since the original samples are authentic images of the highest quality.
Besides, removing the image conditioning substantially degrades performance,
which mirrors our qualitative studies in Figure \ref{fig:rm_guide}. When a 
small noise level ($\eta{=}0.20$) is applied with shape tweaking and 
conditioning, the diversity gain is limited, and distortion caused by shape
tweaking may also remain and backfire, yielding $63\%$ accuracy. Increasing
the noise strength to a moderate level ($\eta{=}0.70$) improves coverage 
while preserving class faithfulness, pushing performance to $67.1\%$, the 
best among diffusion-based rows. Pushing noise to the extreme ($\eta{=}1.0$) 
still delivers reasonable performance ($67.0\%$) when conditioning is enabled. 
This performance is enabled by the greater diversity between the generated
output and the original input, but fidelity to the original is
jeopardized, especially when confronted with rare object types. 
Therefore, generation from maximum noise should be discouraged. Besides,
ablating shape tweaks reduces accuracy to $66.1\%$, confirming that 
geometric variation helps cover different poses and layouts. Substituting
true extra images of the same class (“Real/Oracle”) provides an upper 
bound of $78.0\%$, showing the headroom available with more independent 
samples. Meanwhile, traditional geometric test-time edits (rotations, 
affine warps, color jitter) only reach $57.89\%$, supporting the 
observation that such transforms add little new information and may
distort original images, for example by decreasing image resolution.

% Overall, the ablation highlights two findings: (i) moderate controlled noise balances fidelity and diversity best; and (ii) shape tweaks contribute complementary geometric diversity beyond the stochastic perturbations. Our full configuration (shape tweak + controlled noise + attention-conditioned diffusion) consistently outperforms classical test-time augmentation and approaches the oracle bound.

\section{Limitations and Future Work}
\label{sec:limitation}
The augmentation introduces inference overhead
(see Appendix~\ref{sec:efficiency}). Reducing computational cost
(e.g., via lighter generative backbones or faster denoising schedules) 
is useful, and the overhead may diminish as computational infrastructure
advances. Overall, our main contribution is not the
specific implementation, but the algorithm design that benefits from
computational advances. Additionally, evaluation in this work is limited to image classification and
a training-free protocol. 
Extending the data augmentation operator to other tasks (e.g., training data augmentation and image generation) 
remains a meaningful direction for future work. 

\section{Conclusions}
To conclude, we identify the scarcity of test-time data augmentation in the literature,
and study the usefulness of one-shot test-time 
generative augmentation for generalization improvement. 
Our findings are positive: 1S-DAug, as a simple plug-in
for frozen FSL models, achieves consistent non-trivial
accuracy gains across different datasets without model-parameter access. This
model-agnostic, data-side design makes the
approach practically meaningful for deployments with large, fixed 
or restricted models. Overall, we believe this direction
offers a high-potential building block for inference robustness in
real-world scenarios.

% \section{Ethics Statement}
% Assuming high quality of the pretrained image generator, the data augmentation can improve accuracy and robustness in
% data-scarce settings and under distribution shift, which is particularly relevant in
% high-stakes perception pipelines (e.g., autonomous driving and medical imaging) where
% generalization failures can have serious consequences. By generating task-specific
% variants directly from a readily available pretrained model, the approach reduces both prediction error and the computational and data-collection costs typically associated with retraining,
% fine-tuning, and restricted model parameter access (e.g.,
% API-only or proprietary deployments). 

% In safety-critical deployments, any gains from augmentation should be verified with rigorous testing and human oversight.
\newpage
\section{Use of AI Tools}
Large language models (LLMs) were only used to aid and polish writing (e.g., improving clarity, 
tightening tone, harmonizing mathematical presentations, 
converting prose to \LaTeX, and recording implementation details). The technical content,
mathematical formulations, and experimental design were authored by the authors.

LLMs were \textbf{not} used to generate code, select experimental data; to tune hyperparameters automatically; 
or to produce figures or tables beyond cosmetic wording. 
All code was implemented and executed by human, and all analyses were conducted by the authors; all numbers reported in
the paper come from our runs. 
Where LLM-assisted text appears (e.g., phrasing of method and related-work passages), 
it was reviewed for factual faithfulness and edited for technical precision.

% \paragraph{Accountability.}
% The authors take full responsibility for the content of the paper, including any errors. This disclosure is provided in the interest of transparency and in accordance with community best practices.

\section{Reproducibility Statement}
We take reproducibility seriously. 
The paper specifies 
the experimental setup (encoders, classifier, datasets, and episodic protocol),
and all evaluation details (Section~\ref{sec:experiment}); ablations and qualitative
analyses are provided to validate design choices. Appendix~\ref{sec:details} provides details on
data preprocessing (including upscaling and shape-tweak parameters), hyperparameters 
(noise level, conditioning strength, denoising steps), and the exact episodic sampling 
procedure (5-way-1-shot, 15{,}000 queries, 95\% confidence intervals). Upon acceptance,
we will release code with deterministic seeds to reproduce the results exactly.

%% file: tables/mini_tiered.tex
\begin{table*}[ht]
\centering
\caption{Inductive 5-way-1-shot and 5-way-5-shot accuracy (\%) on miniImagenet and tieredImagenet with Res12 backbones. Parenthetical gains are relative to the corresponding re-implemented baseline. Dashes denote unavailable or unreported results. The best results among our variants and among other FSL methods are highlighted in bold. We do not chase after the state-of-the-art, 
and can likely achieve even better performance using stronger base models.}
\resizebox{\textwidth}{!}{%
\begin{tabular}{lcccc}
\toprule
& \multicolumn{2}{c}{\textbf{5-Way-1-Shot (\%)}} & \multicolumn{2}{c}{\textbf{5-Way-5-Shot (\%)}} \\
\cmidrule(lr){2-3}\cmidrule(lr){4-5}
\textbf{Method (Res12)} & \textbf{miniImagenet} & \textbf{tieredImagenet} & \textbf{miniImagenet} & \textbf{tieredImagenet} \\
\midrule
DeepEMD \citep{zhang2020deepemd}                & $65.91 \pm 0.82$ & $71.16 \pm 0.87$ & $82.41 \pm 0.56$ & \textbf{86.03} $\pm 0.58$ \\
Meta-MaxUp \citep{ni2020data}                   & $62.81 \pm 0.34$ &- & $79.38 \pm 0.24$ &- \\
Meta-Baseline \citep{chen2021metabaseline}      & $63.17 \pm 0.23$ & $68.62 \pm 0.27$ & $79.26 \pm 0.17$ & $83.74 \pm 0.18$ \\
MetaOptNet \citep{lee2019metaoptnet}            & $62.64 \pm 0.61$ & $65.99 \pm 0.72$ & $78.63 \pm 0.46$ & $81.56 \pm 0.53$ \\
% TADAM \citep{oreshkin2018tadam}                 & $58.50 \pm 0.30$ & $-$              & $76.70 \pm 0.30$ & $-$ \\
% MTL \citep{sun2019mtl}                          & $61.20 \pm 1.80$ & $-$              & $75.50 \pm 0.80$ & $-$ \\
SLA-AG \citep{lee2020slaag}                     & $62.93 \pm 0.63$ & $-$              & $79.63 \pm 0.47$ & $-$ \\
ProtoNet + TRAML \citep{li2020traml}            & $60.31 \pm 0.48$ & $-$              & $77.94 \pm 0.57$ & $-$ \\
ConstellationNet \citep{xu2021constellationnet} & $64.89 \pm 0.23$ & $-$              & $79.95 \pm 0.17$ & $-$ \\
Classifier-Baseline \citep{chen2021metabaseline} & \textit{$58.91 \pm 0.23$} & \textbf{$68.07 \pm 0.26$} & \textbf{$77.76 \pm 0.17$} & \textbf{$83.74 \pm 0.18$} \\
DFR \citep{cheng2023disentangled}         & \textbf{67.74} $\pm 0.86$ & $\textbf{71.31} \pm 0.93$ & \textbf{82.49} $\pm 0.57$ & $85.12 \pm 0.64$ \\
\midrule
% ProtoNet-Res12 \citep{snell2017prototypical}          & $62.39 \pm 0.21$          & $68.23 \pm 0.23$ & $80.53 \pm 0.14$          & $84.03 \pm 0.16$ \\
ProtoNet-Res12 (re-impl.) \citep{snell2017prototypical}                           & $60.01 \pm 0.65$ & $65.28 \pm 0.32$ & $75.34 \pm 0.49$ & $81.13 \pm 0.29$ \\
ProtoNet-Res12+1S-DAug-1 (Ours)                      & $62.90 \pm 0.66$ \textit{(+4.82\%\,$\uparrow$)} & $69.06 \pm 0.32$ \textit{(+5.79\%\,$\uparrow$)} & $78.89 \pm 0.48$ \textit{(+4.71\%\,$\uparrow$)} & $83.86 \pm 0.27$ \textit{(+3.36\%\,$\uparrow$)} \\
ProtoNet-Res12+1S-DAug-2 (Ours)                      & $64.61 \pm 0.66$ \textit{(+7.67\%\,$\uparrow$)} & $70.32 \pm 0.32$ \textit{(+7.72\%\,$\uparrow$)} & - & - \\
ProtoNet-Res12+1S-DAug-3 (Ours)                      & $64.94 \pm 0.66$ \textit{(+8.22\%\,$\uparrow$)} & - & - & - \\
% FEAT-Res12 \citep{ye2020fewshot}                      & $66.78 \pm 0.20$          & $70.80 \pm 0.23$ & $82.05 \pm 0.14$ & $84.79 \pm 0.16$ \\
FEAT-Res12 (re-impl.) \citep{ye2020fewshot}                               & $63.31 \pm 0.65$ & $68.28 \pm 0.28$ & $77.90 \pm 0.48$ & $82.21 \pm 0.28$\\
FEAT-Res12+1S-DAug-1 (Ours)                          & $67.08 \pm 0.65$ \textit{(+5.95\%\,$\uparrow$)} & $71.85 \pm 0.28$ \textit{(+5.23\%\,$\uparrow$)} & 81.96 $\pm 0.44$ \textit{(+5.21\%\,$\uparrow$)} & 84.82 $\pm 0.26$ \textit{(+3.17\%\,$\uparrow$)} \\
FEAT-Res12+1S-DAug-2 (Ours)                          & $69.04 \pm 0.65$ \textit{(+9.05\%\,$\uparrow$)} & \textbf{73.18} $\pm 0.27$ \textit{(+7.18\%\,$\uparrow$)} &82.62 $\pm 0.45$ \textit{(+6.06\%\,$\uparrow$)} & \textbf{85.55} $\pm 0.25$ \textit{(+4.06\%\,$\uparrow$)} \\
FEAT-Res12+1S-DAug-3 (Ours)                          & \textbf{69.25} $\pm 0.65$ \textit{(+9.38\%\,$\uparrow$)} & - & \textbf{83.38} $\pm 0.41$ \textit{(+7.03\%\,$\uparrow$)} & - \\
\bottomrule
\end{tabular}%
}
\label{tab:main}
\end{table*}

%% file: tables/cub_animals.tex
\begin{table}[h]
% \vspace{-1.3em}
\caption{Inductive 5-way-1-shot accuracy (\%) on Animals with Res12 backbones and CUB with ConvNet backbones. Parenthetical gains are relative to the corresponding non-augmented baseline.}
\centering
\begin{tabular}{lcc}
\toprule
% & \multicolumn{2}{c}{\textbf{5-Way-1-Shot \%}} \\
% \cmidrule(lr){2-3}
% \midrule
\textbf{Method (Res12/ConvNet)} & \textbf{Animals} & \textbf{CUB}    \\
\midrule
ProtoNet & $73.20 \pm 0.63$ &$46.38 \pm 0.22$             \\
% ProtoNet-Res12 + FUNIT (FSL-Rectifier) \citep{liu}    & $74.08 \pm 0.60$   & -   \\
ProtoNet+1S-DAug-2 (Ours)   & $75.20 \pm 0.65$ \textit{(+2.73\%\,$\uparrow$)}  &$55.50 \pm 0.24$ \textit{(+19.66\%\,$\uparrow$)}      \\
% FEAT-Res12 + FUNIT (FSL-Rectifier)          &- & $80.54 \pm 0.56$   \\
FEAT       &79.37 $\pm 0.59$          &51.10 $\pm 0.24$        \\  
FEAT+1S-DAug-2 (Ours)       &\textbf{80.66} $\pm 0.62 \textit{ (+1.63\%\,$\uparrow$)}$          &\textbf{61.55} $\pm 0.25$ \textit{(+20.45\%\,$\uparrow$)} \\
\bottomrule
\end{tabular}%
% \vspace{-0.4em}
\label{tab:cub}
\end{table}

\input{tables/swin_vit}

%% file: tables/swin_vit.tex
\begin{table}[h]
\centering
\caption{5-way-1-shot accuracy (\%) on miniImagenet/CUB with ViTSmall/SwinTiny backbones. Parenthetical gains are relative to the corresponding non-augmented baseline.}
\begin{tabular}{lccc}
\toprule
\textbf{Dataset} & \textbf{Method} & \textbf{ViTSmall} & \textbf{SwinTiny} \\
\midrule
\multirow{4}{*}{MiniImagenet}
& ProtoNet \citep{snell2017prototypical}        & $71.86$  & $67.32$ \\
& ProtoNet+1S-DAug-1 (Ours)                    & $80.42$ \textit{(+11.91\%\,$\uparrow$)}  & $75.12$ \textit{(+11.59\%\,$\uparrow$)} \\
& ProtoNet+1S-DAug-2 (Ours)                    & $82.76$ \textit{(+15.17\%\,$\uparrow$)} & $77.82$ \textit{(+15.60\%\,$\uparrow$)} \\
& ProtoNet+1S-DAug-3 (Ours)                    & $83.66$ \textit{(+16.42\%\,$\uparrow$)} & 78.92 \textit{(+17.23\%\,$\uparrow$)} \\
\midrule
\multirow{2}{*}{CUB}
& ProtoNet \citep{snell2017prototypical}        & $71.90$  & $69.78$ \\
& ProtoNet+1S-DAug-1 (Ours)                    & $75.72$ \textit{(+5.31\%\,$\uparrow$)}  & $71.76$ \textit{(+2.84\%\,$\uparrow$)} \\
\bottomrule
\end{tabular}%
\label{tab:main2}
\end{table}

%% file: tables/vlm_adversarial.tex
\begin{table}[h]
\centering
\caption{POPE adversarial-split results (\%) with LLaVA-1.5-7B. Parenthetical gains are relative to our evaluated LLaVA-1.5-7B baseline.}
% \resizebox{\columnwidth}{!}{%
\begin{tabular}{lcc}
\toprule
\textbf{Method} & \textbf{Accuracy} & \textbf{F1} \\
\midrule
LLaVA-1.5-7B (ours eval.) & $78.83$ & $77.41$ \\
LLaVA-1.5-7B + VCD \citep{leng2023vcd} & $80.88$ \textit{(+2.60\%\,$\uparrow$)} & \textbf{81.33} \textit{(+5.06\%\,$\uparrow$)} \\
LLaVA-1.5-7B + 1S-DAug-2 (Ours) & $79.57$ \textit{(+0.94\%\,$\uparrow$)} & $77.49$ \textit{(+0.10\%\,$\uparrow$)} \\
LLaVA-1.5-7B + 1S-DAug-4 (Ours) & \textbf{81.13} \textit{(+2.92\%\,$\uparrow$)} & $79.08$ \textit{(+2.16\%\,$\uparrow$)} \\
\bottomrule
\end{tabular}%
% }
\label{tab:vlm_adv}
\end{table}

%% file: tables/vary_shots.tex
\begin{table}[h]
% \vspace{-1.3em}
\caption{Inductive 5-way-1-shot accuracy (mean ± 95\% CI) as a function of the number of augmented copies for supports (rows) and queries (columns).}
\centering
\small
\begin{tabular}{lccccc}
\toprule
& & \multicolumn{4}{c}{\textbf{Query}} \\
\cmidrule(lr){3-6}
\multirow{5}{*}{\textbf{Support}} & {} & \textbf{+0} & \textbf{+1} & \textbf{+2} & \textbf{+3} \\
\cmidrule(lr){2-6}
& +0 & $60.01 \pm 0.65$ & $60.24 \pm 0.69$ & $60.09 \pm 0.70$ & $60.31 \pm 0.70$ \\
& +1 & $60.20 \pm 0.66$ & $62.90 \pm 0.66$ & $63.05 \pm 0.67$ & $63.16 \pm 0.67$ \\
& +2 & $61.80 \pm 0.64$ & $64.53 \pm 0.66$ & $64.61 \pm 0.66$ & $64.87 \pm 0.66$ \\
& +3 & $61.82 \pm 0.64$ & $64.55 \pm 0.65$ & $64.72 \pm 0.66$ & $64.94 \pm 0.66$ \\
\bottomrule
\end{tabular}%
\label{tab:shots}
\end{table}

%% file: tables/ablation.tex
\begin{table}[h]
% \vspace{-1.3em}
\caption{Ablation of aggregation weight adjustment, shape tweak, noise level, and diffusion conditioning. ‘Traditional’ uses standard geometric edits; ‘Real’ substitutes actual additional images of the same object type. Our full setting (shape + controlled noise + attention-conditioned diffusion) outperforms classical test-time augmentation and approaches the oracle (true image) upper bound.}
    \centering
    \resizebox{\columnwidth}{!}{
\begin{tabular}{l l l l | c}
\toprule
\begin{tabular}[c]{@{}l@{}}Same\\Noise Level\end{tabular} &
\begin{tabular}[c]{@{}l@{}}Shape\\Tweak\end{tabular} &
\begin{tabular}[c]{@{}l@{}}Generation\\Techniques\end{tabular} &
\begin{tabular}[c]{@{}l@{}}Additional\\Adjustment\end{tabular} &
\begin{tabular}[c]{@{}c@{}}5-Way-1-Shot\\Accuracy\end{tabular}
\\
\midrule
\tick (0.70)  & \tick & \tick  &\tick (0.3 original image weight) & 66.79 $\pm$ 0.63\\
\tick (0.70) &\tick & \tick  &\tick (remove image conditioning) & 53.94 $\pm$ 0.62\\
 \cross (0.20) & \tick &\tick &\cross &63.45 $\pm$ 0.64\\
\tick (0.70) & \tick & \tick &\cross & 67.08 $\pm$ 0.62 \\
\cross (1.00) & - & \tick &\cross & 67.01 $\pm$ 0.68 \\
\tick (0.70)       & \cross & \tick &\cross & 66.12 $\pm$ 0.66 \\
-       & \tick (Real/Oracle) & \cross &\cross &77.99 $\pm$ 0.62 \\
-       & \tick (Traditional)   & \cross &\cross & 57.89 $\pm$ 0.67 \\
\bottomrule
\end{tabular}
}
\label{tab:ablation}
% \vspace{-1.4em}
\end{table}

%% file: sections/6_appendix.tex
\input{sections/8_diffusion_preliminaries}
\input{sections/7_app_theory}

\section{Efficiency}
\label{sec:efficiency}
\input{tables/efficiency}
We record the wall-clock running time for our inference script, and the results are reported in Table \ref{tab:efficiency}. All experiments are run on a single NVIDIA RTX~A5000 GPU. As expected, runtime scales with the noising level. Higher noising produces larger edits and requires longer inference, whereas lower noising is faster. As we start from the noisy image halfway, the inference cost is generally lower than that of standard diffusion process starting from pure noise. Additionally, we use the vanilla DDPM solver for the main experiments. We experiment with a fast solver, the DPM-Solver++ \citep{lu2025dpm}, which can halve the inference steps without sacrificing accuracy. The 10-step inference time for DPM-Solver++, starting from 0.75 noise level, is 0.23 seconds.

\section{GAN Failure Case Illustration}
Figure \ref{fig:failure} illustrates the failure cases of GAN-based image-to-image translation models \citep{liu2019few}.
\begin{figure}
    \centering
    \includegraphics[width=\linewidth, trim={0, 5cm, 0, 4cm}, clip]{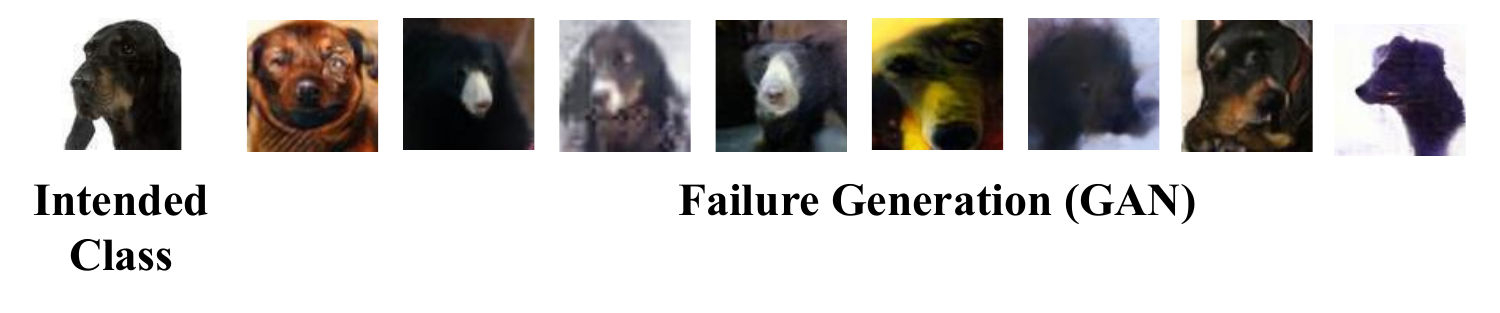}
    \caption{Failure modes of GAN-based image-to-image translation. Examples where image-to-image GAN translation fails to preserve the intended class. Rows contain target class and failed GAN outputs with typical artifacts.}
    \label{fig:failure}
\end{figure}
\section{More Implementation Details}
\label{sec:details}
\subsection{FSL model training (ProtoNet and FEAT)}
We follow the public FEAT repository for episodic training and evaluation, including 5-way, 1-shot meta-training; 15 queries/class for both train and evaluation; Euclidean distance for classification; Res12 as the default backbone. Key arguments (defaults shown where applicable) are exposed by \text{train\_fsl.py}: 
\emph{task setup} \(\{\)dataset=\{\text{miniImagenet}, \text{tieredImagenet}, \text{CUB}, \text{Animals}\}, way=5, shot=1, query=15; eval\_way=5, eval\_shot=1, eval\_query=15\(\}\); 
\emph{optimization} \(\{\)max\_epoch=400, episodes\_per\_epoch=100, num\_eval\_episodes=200, lr \(=10^{-4}\) (with pre-trained weights), lr\_scheduler=\text{step}, step\_size=20, gamma=0.2, momentum=0.9, weight\_decay \(=5\cdot 10^{-4}\)\(\}\); 
\emph{model} \(\{\)model\_class \(\in\) \{\text{ProtoNet}, \text{FEAT}\}, backbone\_class \(\in\) \{\text{ConvNet}, \text{Res12}\}, use\_euclidean (Euclidean distances), temperature=1 (ProtoNet)/64 (FEAT), lr\_mul=10 for the set-to-set head\(\}\). Example FEAT commands for Res12 on \text{tieredImagenet} use \(\text{lr}=2\cdot10^{-4}\), \(\text{step\_size}\in\{20,40\}\), \(\gamma=0.5\), and temperatures \(\text{temperature}=64\), \(\text{temperature2}\in\{64,32\}\); we mirror this recipe for our FEAT runs and use the same episodic protocol for ProtoNet. 

Concretely, in our re-trains we use:
\begin{itemize}
\item \textbf{Backbones:} Res12 (all models except CUB), ConvNet (CUB).
\item \textbf{ProtoNet:} \(\text{model\_class}=\) \text{ProtoNet}, Euclidean distances, max\_epoch \(=400\), episodes\_per\_epoch \(=100\), lr \(=1\text{e-}4\) (pretrained), step scheduler with step\_size \(=20\), \(\gamma=0.2\); temperature \(=1\) unless otherwise tuned on validation; momentum \(=0.9\), weight\_decay \(=5\cdot10^{-4}\). (All other task/eval counts as above.) 
\item \textbf{FEAT:} \(\text{model\_class}=\) \text{FEAT}, lr \(=2\text{e-}4\), lr\_mul \(=10\) for the Transformer head, step scheduler with step\_size in \(\{20,40\}\) and \(\gamma=0.5\); temperature \(=64\), temperature2 \(\in\{64,32\}\); Euclidean distances enabled; same episodic counts as ProtoNet.
\end{itemize}

At evaluation, we sample 15{,}000 queries and report mean accuracy with 95\% confidence intervals, matching the repository’s evaluation practice.

\subsection{1S-DAug (one-shot test-time augmentation) configuration}
We implement 1S-DAug as an image-conditioned denoising pipeline and a light, class-preserving geometric pre‐edit (“shape tweak”). The script exposes the following arguments (defaults in \(\cdot\)), which we fix across all main tables unless the ablation states otherwise:

\paragraph{Core diffusion/editing.}
We use stable-diffusion-v1.5 from the diffusers library as the base image generator.
\(\text{--noise-level}\in[0,1]\) \(\cdot\) \(0.7\): entry point on the diffusion trajectory (larger = more rewrite, smaller = higher faithfulness);
\(\text{--steps}\) \(\cdot\) \(20\): denoising steps;
\(\text{--cfg}\) \(\cdot\) \(9.0\): guidance scale;
\(\text{--seed}\) for deterministic replication; attention slicing is enabled; VAE tiling can be toggled for large images.

\paragraph{Backbone generator.}
\(\text{--model}\) \(\cdot\) \text{runwayml/stable-diffusion-v1-5}; images are fed at \(512\times512\) resolution. 

\paragraph{Image-prompt adapter.}
\(\text{--ip\_adapter}\) (on/off), \(\text{--ip-repo}\) \(\cdot\) \text{h94/IP-Adapter}, \(\text{--ip\_scale}\) \(\cdot\) \(0.8\) controls conditioning strength.

\paragraph{Shape tweaking (geometric seed).}
Enabled via \(\text{--shape\_aug}\); parameters: 
\(\text{--shape-aug-rotate}\) \(\cdot\) \(20^\circ\) (uniform in \([-\!R,+\!R]\)); 
\(\text{--shape-aug-stretch}\) \(\cdot\) \(0.20\) (anisotropic scales \(s_x,s_y\in[1-S,1+S]\)); 
\(\text{--shape-aug-translate}\) \(\cdot\) \(0.025\) (fraction of width/height); 
\(\text{--shape-aug-persp}\) \(\cdot\) \(0.12\) (corner jitter fraction for a single‐view projective warp). 
Intermediate augmented images can be saved for inspection with \(\text{--save\_aug}\).

\paragraph{I/O and batching.}
Single-image mode or directory batch mode; recursive directory traversal and extension override are supported; per-image runtime and peak memory are logged (used in our efficiency table).

\paragraph{Recommended ranges (used in ablations).}
Noise levels \(\in\{0.25,0.5,0.75,1.0\}\);
shape tweaks at the defaults above or slightly weaker for fine-grained datasets. When noise is very small, diversity is limited; when very large, fidelity drops unless image conditioning is active (consistent with our qualitative/quantitative ablations).

\subsection{VLM experiment implementation details}
\label{sec:vlm-details}
We implement the Vision Language Model (VLM) extension with the same 1S-DAug operator but replace the few-shot prototype head by an answer-space aggregation over a frozen LLaVA-1.5-7B model. Our evaluation script is based on \text{ens2.py}. Concretely, we load the local HuggingFace-format \text{llava-1.5-7b-hf} checkpoint with \text{AutoProcessor} and \text{LlavaForConditionalGeneration}, use \text{bfloat16} weights with automatic device placement, and evaluate on the adversarial split of POPE using COCO \text{val2014} images. Unless otherwise stated, we evaluate 3{,}000 examples with seed \(1234\).

\paragraph{Decoding and answer format.}
For each POPE sample, we append the instruction ``Please answer this question with one word.'' and generate a one-token response (\text{paper-match-max-new-tokens}=1). The decoded continuation is parsed into a binary \text{yes}/\text{no} prediction. This keeps the protocol close to POPE and avoids introducing a task-specific classifier head.

\paragraph{View generation and aggregation.}
For an input image, we first obtain the original LLaVA answer, then optionally synthesize \(K\in\{2,4\}\) augmented views using the same 1S-DAug img2img pipeline. Each generated view is queried with the same question and one-word decoding rule. The final decision is produced by weighted voting over the original and augmented answers. In the current implementation, \text{orig-logit-weight}=1.0 and \text{aug-logit-weight}=1.0, and each augmented view receives unit vote weight in the final tally, so the aggregate influence of augmentations increases with \(K\). We report both \(K=2\) and \(K=4\) in the main paper.

\paragraph{Generator settings.}
The VLM extension uses \text{runwayml/stable-diffusion-v1-5} with DPM-Solver++ enabled, \text{steps}=10, \text{cfg}=9.0, \text{noise-level}=0.1, IP-Adapter enabled with \text{h94/IP-Adapter} and \text{ip-scale}=0.8, and shape tweaking turned on for the augmented runs. The shape-tweak parameters are \text{rotate}=10 degrees, \text{stretch}=0.40, \text{translate}=0.01, and \text{persp}=0.10. For the baseline row without augmentation, we set \(K=0\) and disable shape tweaking. The script also records per-image generation time and writes summary metrics in the form \text{aug\_views\_\{K\}\_split\_\{split\}\_shapetweak\_\{flag\}\_num\{N\}.txt}.

\section{More Related Work}
\label{sec:related}
\paragraph{Few-shot Learning.}
FSL methods commonly fall into metric-, model-, and augmentation-based families. 
\emph{Metric-based} methods learn an embedding where queries are classified by proximity to supports or class prototypes, including Matching Networks~\citep{vinyals2016matching}, Prototypical Networks~\citep{snell2017prototypical}, Relation Networks~\citep{sung2018relation}, and episodic feature adaptation such as FEAT~\citep{ye2020fewshot}. Strong baselines refine this recipe with improved training protocols and heads, e.g., Baseline++~\citep{chen2019closer}, Meta-Baseline~\citep{chen2021metabaseline}, MetaOptNet~\citep{lee2019metaoptnet}, and transductive inference methods such as TPN~\citep{liu2019tpn}, LaplacianShot~\citep{ziko2020laplacianshot}, and TIM~\citep{iscen2020tim}. 
\emph{Model-based} approaches emphasize rapid parameter adaptation from few examples, e.g., gradient-based meta-learning with MAML~\citep{finn2017maml}, Meta-SGD~\citep{li2017metasgd}, Reptile~\citep{nichol2018reptile}, and ANIL~\citep{raghu2020anil}. 
\emph{Augmentation-based} approaches increase training diversity via feature or image synthesis—e.g., feature hallucination~\citep{hariharan2017lowshot} and delta-based example synthesizers~\citep{schwartz2018deltaencoder}. These are primarily \emph{training-time} techniques that rely on base-class supervision; by contrast, few-shot \emph{test-time} augmentation must produce high-quality, class-faithful variants for unseen classes without retraining or labels.

Test-time generative augmentation for FSL remains limited. \citet{bai2025fsl} uses an adversarial image-to-image translator to combine the geometric “shape’’ of one image with the class-defining “style’’ of another (i.e., FUNIT \citep{liu2019few}) for inference-time augmentation. While a useful proof of concept, the dataset scope is narrow and failure arises on more complex, diverse categories, reflecting the difficulty of preserving content under large structural gaps.

\paragraph{Diffusion Models.}

Early diffusion models established iterative denoising as a competitive generative paradigm~\citep{sohldickstein2015deep,ho2020ddpm}, with subsequent improvements to training and sampling~\citep{nichol2021improved}. Latent-space diffusion amortizes computation via a learned autoencoder, enabling high-resolution synthesis~\citep{rombach2022highresolution}. Attention-based conditioning adapters inject external signals into cross-attention without retraining the denoiser, supporting controllable editing and image-conditioned generation, including image-prompt adapters~\citep{ye2023ipadapter}, general adapters~\citep{mou2023t2iadapter}, and control modules such as ControlNet~\citep{zhang2023controlnet}. These advances in stability, controllability, and fidelity make diffusion well-suited for few-shot test-time augmentation. Editing-by-denoising constructs variants by adding controlled noise to a source image and running the reverse process with conditioning~\citep{meng2022sdedit}. However, in this setup, too little noise yields small changes, and too much sacrifices faithfulness (e.g., changes object type), which is not suitable for data augmentation.

\paragraph{GAN/Diffusion-based Data Augmentation.}
Adversarial generators have long been used for data expansion and translation. Few-shot translation frameworks (e.g., FUNIT~\citep{liu2019few} and derivatives), unpaired mappers (CycleGAN~\citep{zhu2017unpaired}), and class-conditional generators (StyleGAN families~\citep{karras2019style}) can expand training sets but face limitations for FSL test-time augmentation, including training stability, mode coverage, and faithfulness for \emph{unseen} classes without supervision. For diffusion-based generators, there have been a few recent works focusing on using the diffusion-based synthetic images for downstream tasks other than few-shot learning. However, these works rely on fine-tuning with text prompts or a handful of extra target-class samples, not suitable for test-time augmentation \citep{He2023IsSyntheticDataReady, Benigmim2023DATUM}. In contrast, we do not rely on any label or additional target-class samples, and the strict set-up fulfills the requirement for challenging downstream tasks like FSL test-time augmentation.

\paragraph{Vision-Language Models and Hallucination Mitigation.}
Recent vision-language models have become a widely used paradigm for multimodal reasoning, from modular frozen-backbone systems such as BLIP-2~\citep{li2023blip2} and InstructBLIP~\citep{dai2023instructblip} to instruction-tuned assistants such as LLaVA~\citep{liu2023visual}. Their growing practical relevance motivates training-free test-time methods that remain applicable even when model parameters are fixed or inaccessible. A central reliability challenge is object hallucination, namely the generation of image-inconsistent content. \citet{li2023pope} systematize this issue through POPE, and recent surveys summarize its causes and mitigation strategies~\citep{liu2024surveyhallucination}. Existing reduction methods include decoding-time approaches such as Visual Contrastive Decoding (VCD)~\citep{leng2023vcd} and OPERA~\citep{huang2024opera}, which modify inference to better ground outputs in the visual signal. Our VLM extension is complementary to this line of work: rather than changing the token-level decoding rule alone, we perturb and re-query the visual evidence itself via 1S-DAug, then aggregate the resulting answers across views.

\section{Algorithm Summary}
Algorithmic summary of 1SDAug is available as Algorithm \ref{alg:1s-daug}.

\label{sec:summary}
\begin{algorithm}[t]
\caption{1S-DAug (Single image $x$)}
\label{alg:1s-daug}
\begin{algorithmic}[1]
\REQUIRE image $x$; steps $T$; schedule $(\beta_t)$; user noise $\eta\in[0,1]$; conditioning weight $\lambda_{\mathrm{img}}$; number of variants $K$; 
\STATE \textbf{Geometric seed:} sample a shape tweak $T_{\psi}$ and set $x_{\mathrm{geom}} \leftarrow T_{\psi}(x)$
\STATE \textbf{Noising entry:} compute $t_0$ from $\eta$ via equation~\eqref{eq:eta-map}; draw $x_{t_0} \sim q(\cdot \mid x_{\mathrm{geom}})$ using equation~\eqref{eq:vp-forward}
\STATE \textbf{Working state:} set $z_{t_0} \leftarrow x_{t_0}$ (pixel space) or $z_{t_0} \leftarrow \mathrm{Enc}(x_{t_0})$ (latent variant)
\FOR{$k = 1$ \textbf{to} $K$}
  \FOR{$t = t_0, t_0-1, \dots, 1$}
    \STATE Form $K_t, V_t$ by equation~\eqref{eq:attn-concat} using $x$ (and $p$ if used); set $Q_t = W_Q z_t$ and compute $c_t \leftarrow A_t(Q_t, K_t, V_t)$
    \STATE \textbf{Reverse step:} update $z_{t-1} \leftarrow \mu_{\varphi}(z_t, t, c_t) + \sigma_t \epsilon$ via equation~\eqref{eq:reverse}, with $\epsilon \sim \mathcal{N}(0, I)$
  \ENDFOR
  \STATE \textbf{Decode:} $\tilde{x}^{(k)} \leftarrow \mathrm{Dec}(z_0)$ \COMMENT{identity if denoising in pixel space}
\ENDFOR 
\textbf{Return} $\{\tilde{x}^{(k)}\}_{k=1}^K$
\end{algorithmic}
\end{algorithm}

\section{1S-DAug Visualization}
\label{sec:viz}
Figure \ref{fig:viz1} and Figure \ref{fig:viz2} illustrate more image generation results of our proposed method.
\begin{figure}[ht]
    \centering
    \includegraphics[width=\linewidth]{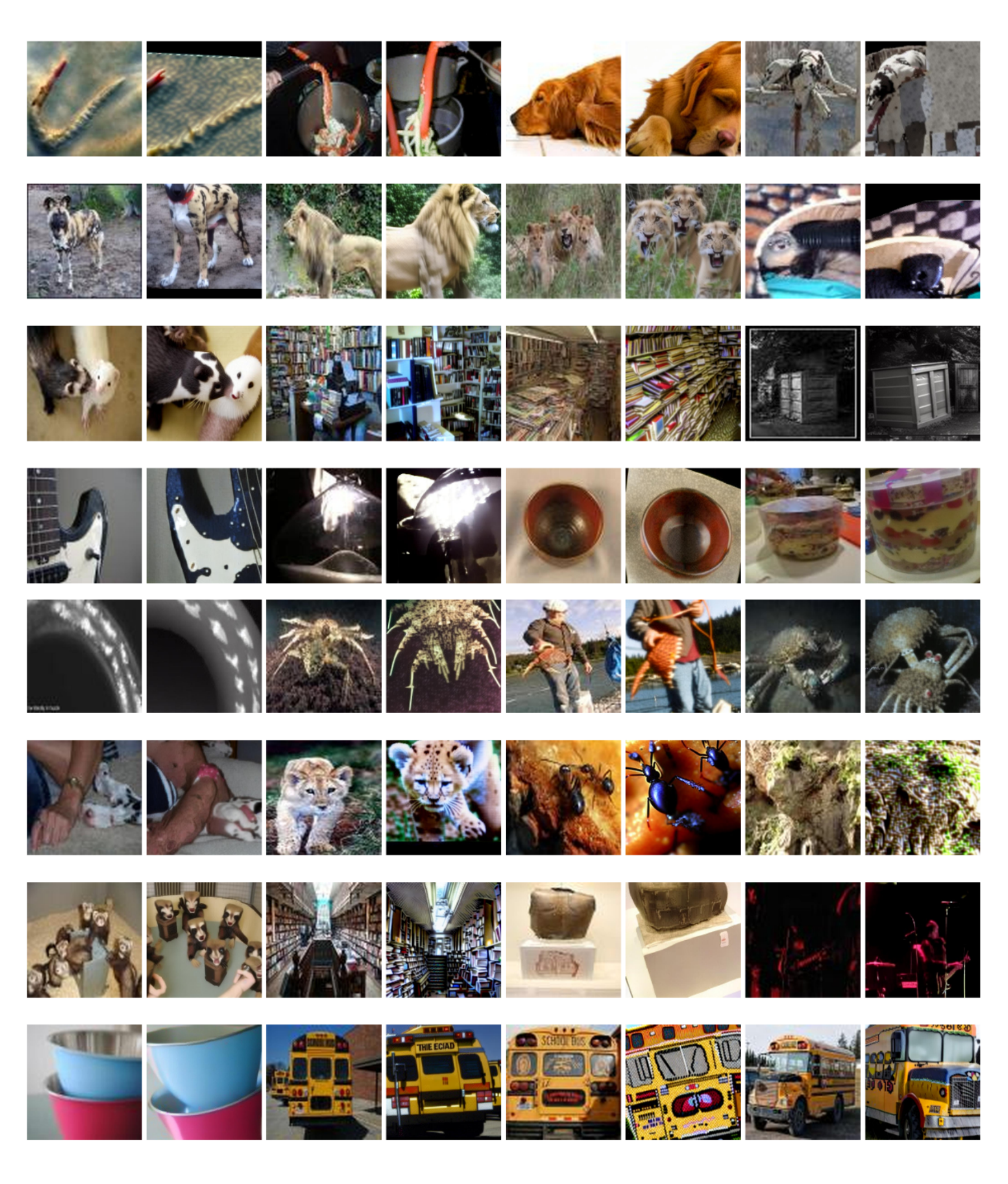}
    \caption{More qualitative results from 1S-DAug. Each pair contains the original image followed by our synthesis. All visualization pairs are random without cherry-picking.}
    \label{fig:viz1}
\end{figure}
\begin{figure}
    \centering
    \includegraphics[width=\linewidth]{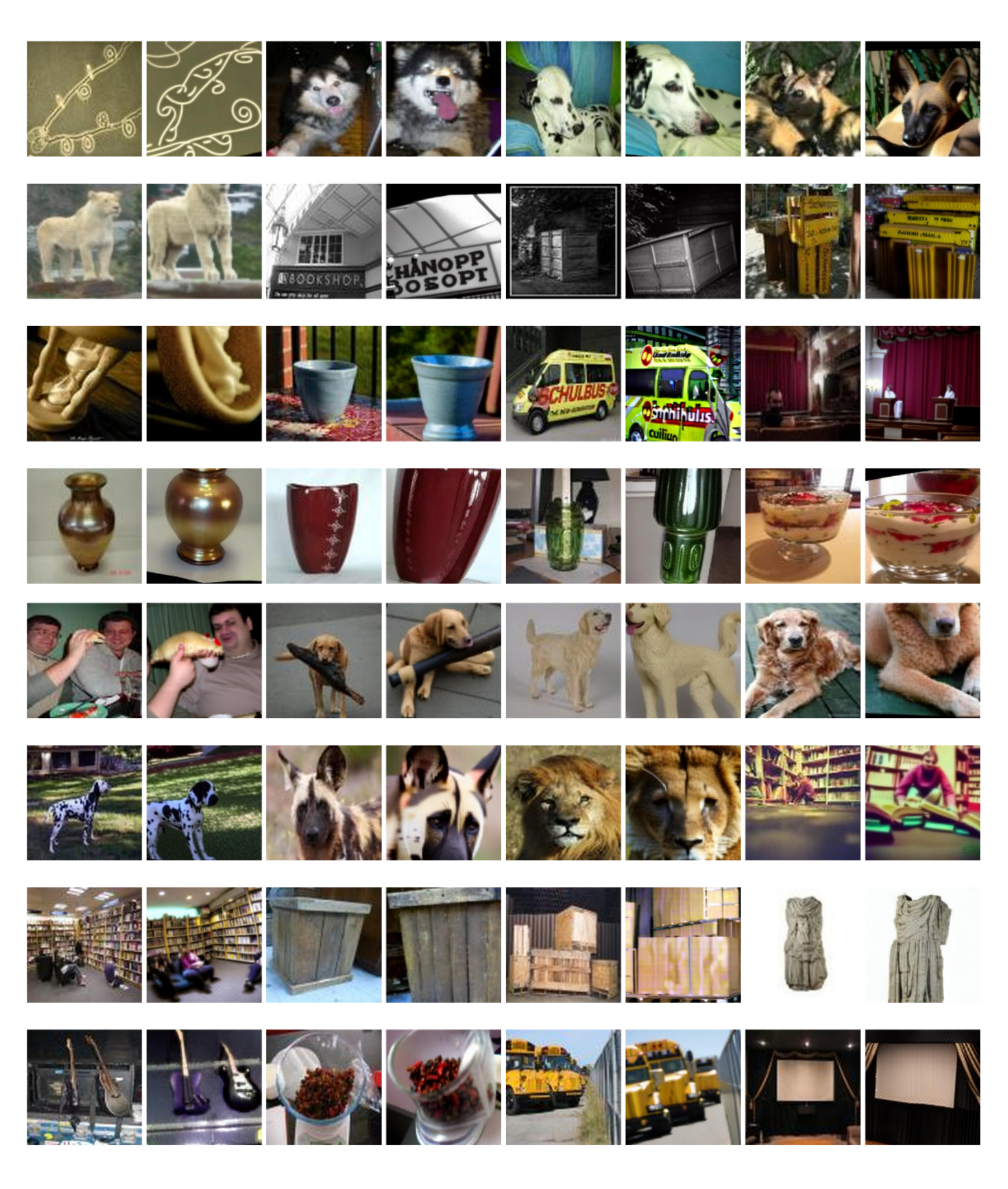}
    \caption{More qualitative results from 1S-DAug. Each pair contains the original image followed by our synthesis. All visualization pairs are random without cherry-picking.}
    \label{fig:viz2}
\end{figure}

%% file: sections/8_diffusion_preliminaries.tex
\section{Diffusion Preliminaries}
\label{app:diffusion-preliminaries}

Let $x_0\in\mathcal{X}$ be an image and $x_t$ its noised version at step
$t\in\{1,\dots,T\}$. A variance schedule $(\beta_t)_{t=1}^T$ defines
$\alpha_t = 1-\beta_t$ and
$\bar{\alpha}_t = \prod_{\tau=1}^t \alpha_\tau$. We use the
variance-preserving (VP) forward process
\citep{sohldickstein2015deep,ho2020ddpm,nichol2021improved,meng2022sdedit}
\begin{equation}
q(x_t \mid x_0)
=
\mathcal{N}\!\big(\sqrt{\bar{\alpha}_t}\,x_0,\ (1-\bar{\alpha}_t)\mathbf{I}\big),
\label{eq:vp-forward}
\end{equation}
so that sampling $x_t$ can be written as
$x_t=\sqrt{\bar{\alpha}_t}\,x_0+\sqrt{1-\bar{\alpha}_t}\,\epsilon$ with
$\epsilon\sim\mathcal{N}(0,\mathbf{I})$.

A user noise level $\eta\in[0,1]$ is mapped to a discrete start time
$t_0\in\{1,\dots,T\}$ by matching cumulative noise:
\begin{equation}
t_0=\arg\min_{t\in\{1,\dots,T\}}\big|(1-\bar{\alpha}_t)-\eta^2\big|.
\label{eq:eta-map}
\end{equation}

Let $\hat\epsilon_\varphi$ be a learned noise predictor with parameters
$\varphi$, and let $c_t$ be the conditioning signal at time $t$. The VP
reverse update is
\begin{equation}
x_{t-1} = \mu_\varphi(x_t,t,c_t) + \sigma_t\,\epsilon,
\label{eq:reverse}
\end{equation}
with $\epsilon\sim\mathcal{N}(0,\mathbf{I})$ and sampler noise level
$\sigma_t$ (we use $\sigma_t{=}0$ in our experiments), where
$\mu_\varphi(x_t,t,c_t)
= \alpha_t^{-1/2}\!\big(x_t - \tfrac{1-\alpha_t}{\sqrt{1-\bar{\alpha}_t}}\,
\hat\epsilon_\varphi(x_t,t,c_t)\big)$.
The same formulas apply in a latent space $z_t$ via an encoder--decoder
$(\mathrm{Enc},\mathrm{Dec})$, with denoising output
$\tilde{x}=\mathrm{Dec}(z_0)$~\citep{rombach2022highresolution}.

%% file: sections/7_app_theory.tex
\section{Proof of the Risk-Decomposition Proposition}
\label{app:risk-proof}

\begin{proof}
Recall that $\mathcal R(g) = \tfrac{1}{4}\,\mathbb{E}\bigl[(g(x)-y)^2\bigr]$ and
$f,f_A:\mathcal{X}\to\{-1,1\}$, $\tilde f = \tfrac12(f+f_A)$.
Since $f^2(x)=f_A^2(x)=y^2=1$,
\begin{equation}
\mathcal R(f)
= \frac{1}{4}\,\mathbb{E}\bigl[(f-y)^2\bigr]
= \frac{1}{4}\,\mathbb{E}[f^2 - 2fy + y^2]
= \frac{1}{2} - \frac{1}{2}\,\mathbb{E}[f(x)y].
\end{equation}
For $\tilde f$,
\begin{equation}
\mathcal R(\tilde f)
= \frac{1}{4}\,\mathbb{E}\Bigl[\Bigl(\frac{f+f_A}{2} - y\Bigr)^2\Bigr]
= \frac{1}{16}\,\mathbb{E}\bigl[(f+f_A)^2\bigr]
  - \frac{1}{4}\,\mathbb{E}\bigl[(f+f_A)y\bigr]
  + \frac{1}{4}\,\mathbb{E}[y^2].
\end{equation}
Expanding $(f+f_A)^2 = f^2 + 2ff_A + f_A^2$ and using $f^2=f_A^2=y^2=1$,
\begin{equation}
\begin{aligned}
\mathcal R(\tilde f)
&= \frac{1}{16}\,\mathbb{E}[2 + 2ff_A]
  - \frac{1}{4}\,\mathbb{E}[fy + f_Ay]
  + \frac{1}{4} \\
&= \frac{1}{8}\bigl(1 + \mathbb{E}[f(x)f_A(x)]\bigr) \\
&\quad
  - \frac{1}{4}\,\mathbb{E}[f(x)y]
  - \frac{1}{4}\,\mathbb{E}[f_A(x)y]
  + \frac{1}{4}.
\end{aligned}
\end{equation}
Subtracting $\mathcal R(f)$ gives
\begin{equation}
\mathcal R(\tilde f) - \mathcal R(f)
= \frac{1}{4}\bigl(\mathbb{E}[f(x)y] - \mathbb{E}[f_A(x)y]\bigr)
 + \frac{1}{8}\bigl(\mathbb{E}[f(x)f_A(x)] - 1\bigr).
\end{equation}
\end{proof}

% Throughout this appendix we use the pairwise reduction.
% Each input is a query-prototype pair $x = (q,p)$ with label $y \in \{-1,1\}$.
% For a fixed encoder $\Phi_\theta$ and prototype $p$ we define the difference
% feature $\Omega_\theta(x) := \Phi_\theta(q)-p$.
%  and the shifted Euclidean score
% $g_\theta(x) := \beta-\|\Omega_\theta(x)\|_2^2$ for a fixed decision threshold $\beta\in\mathbb{R}$.
% The $0$-$1$ pairwise risk of $\theta$ is
% \[
% R_{\mathrm{cls}}(\theta) := \mathbb{P}\bigl(y\,g_\theta(x)\le 0\bigr),
% \]
% which serves as a proxy for the actual rank-based episodic classification.

% Where the arguments apply to generic real-valued predictors, we write $g$
% for a function in a class $\mathcal{G}$ and specialize to the encoder
% score class $\mathcal{G} = \{x\mapsto g_\theta(x):\theta\in\Theta\}$ at the end.

\section{Proof of the Radius-Reduction Proposition}
\label{app:radius-reduction}
\begin{proposition}[Radius reduction with one augmentation]
Assume $\{\Omega_i^{(0)},\Omega_i^{(1)}\}_{i=1}^{m}$ are from the same continuous
distribution on $\mathbb{R}^d$. Let
\begin{equation}
\bar{\Omega}_i := \frac{1}{2}\bigl(\Omega_i^{(0)}+\Omega_i^{(1)}\bigr),\;
\hat r^{\max} := \max_i \|\bar{\Omega}_i\|_2,\;
r^{\max}_{\mathrm{aug}} := \max_i \|{\Omega}_i^{(1)}\|_2,\;
r_{\mathrm{orig}}^{\max} := \max_i \|\Omega_i^{(0)}\|_2.
\end{equation}
Then
\begin{equation}
\mathbb{P}\bigl(\hat r^{\max} < r_{\mathrm{orig}}^{\max}\bigr) > 0.5.
\end{equation}
\end{proposition}

\begin{proof}
Since $\{\Omega_i^{(0)}\}_{i=1}^m$ and $\{\Omega_i^{(1)}\}_{i=1}^m$ are from the 
same distribution, then by symmetry,
\[
\mathbb{P}(r^{\max}_{\mathrm{aug}}<r^{\max}_{\mathrm{orig}})=\mathbb{P}(r^{\max}_{\mathrm{orig}}<r^{\max}_{\mathrm{aug}})=\tfrac{1}{2}.
\]
Case 1: On the event $r^{\max}_{\mathrm{aug}}<r^{\max}_{\mathrm{orig}}$, we have $\|\Omega_i^{(1)}\|_2<r^{\max}_{\mathrm{orig}}$ for every $i$, while by
definition $\|\Omega_i^{(0)}\|_2\le r^{\max}_{\mathrm{orig}}$ for every $i$. Therefore, for every $i$,
\[
\|\bar{\Omega}_i\|_2
=
\left\|\frac{\Omega_i^{(0)}+\Omega_i^{(1)}}{2}\right\|_2
\le
\frac{\|\Omega_i^{(0)}\|_2+\|\Omega_i^{(1)}\|_2}{2}
<
r^{\max}_{\mathrm{orig}},
\]
where the inequality uses the triangle inequality.

Case 2: When the new augmentation set maximum feature radius is larger, 
it is obvious that there are possible events that either increase or decrease 
the aggregated 
feature radius under this case; the augmentation set maximum feature radius 
may end up pairing a small-magnitude original feature radius, leading to 
decreased maximum feature radius after the averaging. Since that the 
maximum feature radius decreases for sure in Case 1, and does not 
increase for sure in Case 2, the probability of feature radius decrease 
after the averaging is strictly more than 0.5. 

Note that the Case 2 
event with a sufficiently-small-radius datapoint exists for sure in the 
probability space, 
leading to the strict inequality when combined with Case 1 (i.e., 
0.5 $\times$ 1 (Case 1) + 0.5 $\times$ p (Case 2) $>$ 0.5 since $p>0$). 
\end{proof}
\begin{remark}
    When there is more than one augmentation set, we can treat the aggregated set as the augmentation set altogether, and the probability of maximum feature radius increase becomes even lower due to the expected variance reduction (i.e., extremely large feature radius) within the augmentation set.
\end{remark}

%% file: tables/efficiency.tex
\begin{table}[h]
\centering
% \vspace{-1.3em}
\caption{Per-image generation time across noise levels. Higher noise entails more denoising compute. Measured on a single GPU; see experimental setup for hardware details.}
\small
% \resizebox{\columnwidth}{!}{
\begin{tabular}{lcccc}
\toprule
&\textbf{0.25 Noise} & \textbf{0.50 Noise} & \textbf{0.75 Noise} & \textbf{1.00 Noise} \\
% \cmidrule(lr){2-5}
\midrule
\textbf{Generation Time (s) $\downarrow$} &0.41 &0.68 &0.92 &1.42 \\
\bottomrule
\end{tabular}
% }
\label{tab:efficiency}
% \vspace{-0.4em}
\end{table}